%% file: main.tex
\documentclass[twoside,11pt]{article}

\usepackage{times}
\usepackage{listings}
\usepackage{lscape}
\usepackage[abbrvbib,preprint]{jmlr2e}
\usepackage{natbib,subfigure}

\input{./command_arxiv}

\usepackage{hyperref}
\usepackage{url}
\usepackage{listings}
\usepackage{pifont}
\usepackage{multirow}
\usepackage{wrapfig}
\usepackage{ragged2e}
\usepackage[most]{tcolorbox}
\definecolor{darkerlogocolor}{RGB}{20, 0, 145}
\newtcolorbox{ttcolorbox}[1][]{colframe=darkerlogocolor, colback=darkerlogocolor!4!white, title=#1}

\newcommand{\piRef}{\pi_{\mathrm{ref}}}

\newcommand{\Prob}{\mathbb{P}}

\usepackage{geometry}
\usepackage{lastpage}
\jmlrheading{1}{2024}{1-\pageref    {LastPage}}{4/00}{10/00}{Zhang et al}{Zhen-Yu Zhang.}

\ShortHeadings{Data-dependent Exploration for Online RLHF}{Zhang, Tang, Zhang, Ma and Sugiyama}
\firstpageno{1}

\begin{document}

\title{Data-dependent Exploration for Online Reinforcement Learning\\ from Human Feedback}

\author{\name Zhen-Yu Zhang \email zhen-yu.zhang@riken.jp \\
     \addr Center for Advanced Intelligence Project, RIKEN \AND
     \name Yuting Tang \email yuting.tang1996@gmail.com \\
     \addr Graduate School of Frontier Sciences, The University of Tokyo \AND 
     \name Jiandong Zhang \email zhang.jiando@northeastern.edu \\
     \addr Northeastern University \AND
     \name Lanjihong Ma \email malanjihong@zjgsu.edu.cn \\
     \addr Zhejiang Gongshang University
     \AND
     \name Masashi Sugiyama \email sugi@k.u-tokyo.ac.jp\\
     \addr Center for Advanced Intelligence Project, RIKEN\\
     Graduate School of Frontier Sciences, The University of Tokyo
     }

\editor{}

\maketitle

\begin{abstract}
Online \emph{reinforcement learning from human feedback}~{(RLHF)} has emerged as a promising paradigm for aligning \emph{large language models}~{(LLMs)} by continuously collecting new preference feedback during training. A foundational challenge in this setting is \emph{exploration}, which requires algorithms that enable the LLMs to generate informative comparisons that improve sample-efficiency in online RLHF. Existing exploration strategies often derive bonuses via on-policy expectations, which are difficult to estimate reliably from the limited historical preference data available during training; as a result, the policy can prematurely down-weight under-explored regions that may contain high-value behaviors. In this paper, we propose \emph{data-dependent exploration for preference optimization}~{(DEPO)}, a simple and scalable method that leverages historical data to construct an extra uncertainty bonus for high-uncertainty regions, encouraging exploration toward potentially high-value data. Theoretically, we provide a \emph{data-dependent} regret bound for the proposed algorithm, showing that it adapts to the hardness of the learning task itself and can be tighter than worst-case bounds in practice. Empirically, the proposed method consistently outperforms strong baselines across benchmarks, demonstrating improved sample efficiency.
\end{abstract}

\input{./section/introduction}

\input{./section/preliminary}

\input{./section/method}

\input{./section/experiments}
\input{./section/related}

\section{Conclusion}
\label{sec:conclusion}

In this paper, we study the exploration mechanism in online RLHF, where preference feedback is collected on-the-fly and the quality of alignment depends critically on how effectively the learner explores the response space. We identify a practical gap in prior optimism-based exploration approaches: exploration bonuses are often derived from on-policy expectations, while in realistic online DPO pipelines the learner primarily relies on historically collected comparisons with limited coverage. To handle this issue, we propose data-dependent exploration for preference optimization (DEPO), which derives an upper-confidence exploration mechanism directly from historical preference comparisons in a representation space and injects it into the preference optimization objective. The resulting bonus increases exploration for under-covered directions where data are insufficient, while remaining lightweight and scalable for online RLHF training. We provide a data-dependent preference-regret bound that adapts to the geometry of the collected data, and experiments show improvements over strong online RLHF baselines across benchmarks and iterations.

\bibliography{./Refs}

\newpage

\appendix
\input{./section/appendix.tex}

\end{document}

%% file: command_arxiv.tex
\usepackage{lipsum,booktabs}
\usepackage{amsmath,mathrsfs,amssymb,amsfonts,bm}
\usepackage{bbm, dsfont}
\usepackage{paralist}
\usepackage{rotating}
\usepackage{pdflscape}
\usepackage{hyperref,url,dsfont,nicefrac}
\hypersetup{
    colorlinks,
    breaklinks,
    linkcolor = blue,
    citecolor = blue,
    urlcolor  = blue,
}
\allowdisplaybreaks[2]
\usepackage{appendix}
\usepackage{multirow,makecell}
\usepackage{graphicx,color} % more modern
\usepackage{caption2}
\usepackage{standalone}     
\usepackage{preview}
\usepackage{xcolor}
\usepackage{setspace}
\usepackage{lmodern}
\usepackage{algorithmic,algorithm}
\renewcommand{\algorithmicrequire}{ \textbf{Input:}}
\renewcommand{\algorithmicensure}{ \textbf{Output:}}

\renewcommand{\tilde}{\widetilde}
\renewcommand{\hat}{\widehat}

\def \D {\mathcal{D}}
\def \E {\mathbb{E}}

\def \R {\mathbb{R}}

\def \X {\mathcal{X}}
\def \Y {\mathcal{Y}}

\def \x {\mathbf{x}}

\usepackage{inconsolata} % modern mathtt

\def \epsilon {\varepsilon}

\usepackage{mathtools}
\DeclareMathOperator*{\argmax}{arg\,max}
\DeclareMathOperator*{\argmin}{arg\,min}

\let\proof\relax
  
\usepackage{amsthm}
\let\proof\relax

\newenvironment{proof}{\par\noindent{\textbf{Proof}\ }}{\hfill\BlackBox\\[2mm]}

\newtheorem{myThm}{Theorem}

\newtheorem{myLemma}{Lemma}
\newtheorem{myCor}{Corollary}

\theoremstyle{definition}
\newtheorem{myAssum}{Assumption}
\newtheorem{myDef}{Definition}

\newtheorem{myRemark}{Remark}

%% file: section/introduction.tex
\section{Introduction}

\emph{Reinforcement learning from human feedback} (RLHF) has achieved remarkable success in aligning \emph{large language models} (LLMs) with human preferences~\cite{ouyang2022training,touvron2023llama}. Most existing RLHF pipelines operate in an offline setting, where the final alignment quality is largely constrained by the coverage and quality of a pre-collected preference dataset~\cite{ouyang2022training,schulman2017proximal,ZhengmaoZHU2025194347}. To mitigate this limitation, \emph{online RLHF} has recently emerged as an increasingly important paradigm: the model can interactively query new feedback during training, continually improving alignment as it explores and collects new preference data~\cite{dong2024rlhf,cen2024value,NeurIPS'25:rlhf}.

A central challenge in online RLHF is \emph{exploration}: generating informative candidate responses and querying human feedback on them to improve preference learning~\cite{dwaracherla2024efficient,xie2024exploratory,XiaoMA2025194313,guo2024direct}. While \emph{passive} exploration that relies on the stochasticity of the current LLM policy has shown empirical benefits~\cite{dwaracherla2024efficient,guo2024direct}, it can be statistically inefficient and is provably sub-optimal in challenging domains where high-quality trajectories are unlikely to appear by chance~\cite{xie2024exploratory,RenJianWANG20262002320}. This motivates \emph{active} exploration strategies that guide data collection toward uncertain yet potentially valuable regions by incorporating a plug in exploration bonus into the online RLHF objective~\cite{xie2024exploratory,cen2024value,chen2025avoiding}. Intuitively, these methods augment standard RLHF training objective with an exploration component, assigning larger training emphasis to data that are under-explored or hard to distinguish under the current policy. This improves both theoretical convergence and empirical performance over purely passive baselines.

Despite these advances, existing exploration strategies often define bonuses through on-policy expectations, which are difficult to estimate reliably when only historical preference data with limited coverage of the response space are available. In practice, this mismatch makes the exploration signal noisy: once certain regions are judged suboptimal based on sparse feedback, the policy may quickly reduce their sampling probability, even if those regions contain potentially high-value but under-explored behaviors~\cite{cen2024value,xie2024exploratory}. This is particularly problematic in online RLHF, where preference supervision comes from historical data collected along the training trajectory. Without a data-dependent coverage measure, feedback may be used inefficiently, and misestimated behaviors may remain uncorrected for many rounds. This motivates exploration bonuses that are explicitly \emph{data-dependent} and track which parts of the space remain under-covered by historical comparisons, directing exploration to where additional feedback is most needed.

In this paper, we propose \emph{data-dependent exploration for preference optimization} (DEPO), a simple and scalable reward-bias mechanism for online RLHF. Our core idea is to derive an \emph{upper-confidence bound}~{(UCB)}~\cite{lai1985asymptotically,auer2002finite} exploration bonus directly from historical preference comparisons in a representation space, and inject it into the preference optimization objective. Specifically, we represent each response pair with a feature vector and track how well different directions in this feature space have been covered by past comparisons. This allows us to compute an uncertainty-aware data-dependent confidence radius that is larger for under-explored regions, and to use it to encourage continued exploration where the current data are insufficient. The proposed method is lightweight to implement, supports efficient online updates, and is compatible with large-scale LLM training. We further provide a data-dependent regret bound showing that the guarantee adapts to the representation geometry induced by the collected trajectory and can be tighter than worst-case bounds. Experiments demonstrate consistent improvements over several state-of-the-art baselines.

%% file: section/preliminary.tex
\section{Preliminary}

In this section, we introduce the notations and assumptions for RLHF, and briefly review two-stage RLHF, direct preference optimization, and recent advances in RLHF exploration.

Let $\Pi$ be a class of stochastic LLM policies, and let $\x \sim \rho$ be a prompt drawn from a probability distribution $\rho$ over a prompt space $\X$. Each $\pi\in\Pi$ maps a prompt $\x\in\X$ to a conditional distribution over responses $y\in\Y$, denoted by $\pi(\cdot\mid\x)$, where $\Y$ is the response space. Given a prompt $\x$, we sample two responses $y\sim\pi^{(1)}(\cdot\mid\x)$ and $y'\sim\pi^{(2)}(\cdot\mid\x)$, and query a preference oracle that outputs an indicator of whether $y$ is preferred to $y'$, written as $y\succ y'$.

Following prior work, we model the preference oracle using the Bradley-Terry model~\cite{bradley1952rank}, and introducing some basic assumptions.
\begin{myAssum}[Bradley-Terry model~\cite{bradley1952rank}]
\label{ass:BT_model}
There exists an underlying reward function $r^*:\X\times\Y \mapsto \R$ such that for any $(\x,y,'y)\in\X\times\Y\times\Y$, we have
\begin{equation*}
\label{eq:bt_model}
\Prob^*(y\succ y' \mid \x)=\sigma(r^*(\x,y) - r^*(\x,y')),
\end{equation*}
where $\Prob^*(y \succ y' \mid \x)$ represents the probability that $y$ is preferred to $y'$ given $\x$ and $\sigma(\cdot)$ represents the sigmoid function $\sigma(u)=(1+e^{-u})^{-1}$. Without loss of generality, we assume $r^*(\x,y) \in [0, R_{\rm max}]$, $R_{\rm max}\geq 0$.
\end{myAssum}

\begin{myAssum}[Linear Realizability]
\label{ass:linear_rep}
There exists a known feature map $\phi:\mathcal{X}\times\mathcal{Y}\to\mathbb{R}^d$ and an unknown parameter $\theta^*\in\mathbb{R}^{d}$ with $\|\theta^*\|_2 \le S$ such that $r^*(x,y) = \langle \theta^*, \phi(x,y)\rangle$. We define the pairwise feature $\psi(x,y,y')=\phi(x,y)-\phi(x,y')$ with $\|\psi(x,y,y')\|_2\le 1$.
\end{myAssum}

\begin{myAssum}[Finite Reward Class]
\label{ass:finite}
The candidate reward class $\mathcal{R}$ is finite with $|\mathcal{R}| < \infty$, and $r^* \in \mathcal{R}$.
\end{myAssum}

Following the prior work~\cite{chen2025avoiding}, we use \emph{preference regret} as the performance measure. Let $\pi^\star$ denote a comparator policy that is optimal under the true reward function $r^\star$. The preference regret over a sequence of $T$ rounds is defined as
\begin{equation*}
\label{eq:regpref}
\begin{aligned}
& \mathrm{R}_{\mathrm{pref}}(T) = \!\! \sum\limits_{t=1,\ldots,T} \!\! \E_{\substack{\x\sim \rho, y^* \sim \pi^*,\\ y_t \sim \pi_t, y'_t \sim \pi_t^{\mathrm{sam}}}} \!\!\left[P^*(y^* \succ y'_t\mid \x) - P^*(y_t \succ y'_t\mid \x)\right].
\end{aligned}
\end{equation*}

\paragraph{Reinforcement Learning from Human Feedback} In the two-stage RLHF framework~\cite{christiano2017deep,ouyang2022training}, we have an offline preference dataset $D=\{\x_i,y_i^{\mathrm{w}}, y_i^{\mathrm{l}}\}_{i=1}^N$. Given the dataset, we estimate the reward function via \emph{maximum likelihood estimation}~{(MLE)}:
\begin{equation*}
\label{eq:mle}
\begin{aligned}
\hat{r}
& =\argmin_{r\in\mathcal{R}}\sum_{i=1}^N -\log \sigma\big(r(\x_i,y_i^{\mathrm{w}})-r(\x_i,y_i^{\mathrm{l}})\big),
\end{aligned}
\end{equation*}
where $\mathcal{R}$ is the reward function space.

With the learned reward function, the objective of RLHF is to fine-tune the policy $\pi$ to maximize the reward. Following prior theoretical works on RLHF, we consider a Kullback-Leibler (KL)-regularized reward objective, that is,
\begin{equation*}
\label{eq:ppo}
\begin{aligned}
\hat{\pi}
& =\argmax_{\pi\in\Pi}\E_{\;\x\sim\rho, y\sim\pi(\cdot\mid \x)} \big[\hat{r}(\x,y) - \beta \log\frac{\pi(y\mid \x)}{\pi_{\rm ref}(y\mid\x)}\big],
\end{aligned}
\end{equation*}
where $\beta>0$ controls the strength of the KL regularization, and $\pi_{\rm ref}\in\Pi$ is a fixed reference policy.

\paragraph{Direct Preference Optimization} Since optimizing the reward function can be hard, the DPO algorithm is an alternative approach that bypasses the need for explicitly learning the reward function~\cite{rafailov2023direct}. DPO optimizes the policy directly with preference data, i.e.,
\begin{equation}
\begin{aligned}
\label{eq:dpo}
\hat{\pi}_{\rm DPO} & = \argmax_{\pi\in\Pi} L_{\rm DPO}(\pi,D)\\
& = \argmax_{\pi\in\Pi} \sum_{i=1}^{N}\log \sigma\left( \beta m_\pi(\x;y_i^{\mathrm{w}},y_i^{\mathrm{l}}) \right),
\end{aligned}
\end{equation}
where $m_\pi(\x;y,y') = \log\frac{\pi(y\mid \x)}{\pi_{\mathrm{ref}}(y\mid \x)} - \log\frac{\pi(y'\mid \x)}{\pi_{\mathrm{ref}}(y'\mid \x)}$.

\paragraph{Active Exploration in RLHF} Active exploration is important in RLHF, as it encourages LLM policies to generate more diverse and higher-reward responses~\cite{dwaracherla2024efficient}. Motivated by this, a growing body of work has developed exploration-enhanced algorithms that incorporate optimism-based principles into DPO~\cite{cen2024value,xie2024exploratory,chen2025avoiding}. Although their implementations differ, the essence of their methods is to add an \emph{exploration bonus} to the training objective, i.e.,
\begin{equation*}
\label{eq:dpo_exp}
\hat{\pi}\in\argmax_{\pi\in\Pi}
\Big\{L_{\rm DPO}(\pi,D)
\;+\;\alpha\,G(\pi)
\Big\}
\end{equation*}
where $\alpha>0$ and $G(\pi)$ is the exploration bonus. For example, one approach introduces an implicit global optimism term~\cite{xie2024exploratory}, whereas another samples online data using an exploration bonus derived from a calibrated policy~\cite{cen2024value}. To avoid the $\exp(R_{\max})$ dependence in earlier regret bounds, a subsequent line of work instead optimizes preference regret and introduces the exploration bonus
\begin{equation*}
\label{eq:popo_bonus}
G(\pi) = G_{\rm POPO}(\pi) = \E_{\substack{\;\x\sim \rho, y \sim \pi,\\ y'\sim\pi_{\rm sam}}} \left[ \sigma(\beta\,m_\pi(\x;y,y')) \right].
\end{equation*}

%% file: section/method.tex
\section{Algorithm and Theory}

In this section, we present our data-dependent exploration bonus for online RLHF, its representation-based construction, an efficient implementation, and the corresponding theoretical analysis.

\subsection{Data-dependent Exploration for Online RLHF}
Online RLHF, including online DPO, mitigates the coverage and quality limitations of fixed preference datasets in offline RLHF by collecting preference feedback on newly generated responses during training~\cite{dwaracherla2024efficient,song2024importance}. We consider an online RLHF process over a horizon of $T$ rounds. At each round $t \in \{1, \ldots, T\}$, a prompt $\x_t \sim \rho$ is sampled, two candidate responses $y_t$ and $y'_t$ are generated, and the oracle is queried for a preference label on them, yielding an ordered pair $(\x_t, y_t^{\mathrm{w}}, y_t^{\mathrm{l}})$. The learner then uses this data to update the current policy $\pi_{\theta_t}$.

A growing line of work introduces exploration bonuses to encourage more informative comparisons~\cite{cen2024value,xie2024exploratory,chen2025avoiding}, but these bonuses are typically defined through on-policy expectations and are therefore difficult to estimate accurately from historical data with limited coverage. Under Assumption~\ref{ass:BT_model}, this uncertainty is naturally tied to the unknown reward gap $\Delta r^\star(\x;y,y')=r^\star(\x,y)-r^\star(\x,y')$, since $\Prob^\star(y\succ y'\mid \x)=\sigma(\Delta r^\star(\x;y,y'))$. Therefore, we introduce a data-dependent radius $b_t(\x,y,y')$ that measures how uncertain the learner is about this reward gap after observing the historical comparisons, based on an \emph{upper-confidence bound} (UCB) principle~\cite{lai1985asymptotically,auer2002finite}.

Specifically, our goal is to build a data-dependent radius $b_t(\x,y,y')$ such that, with high probability,
\begin{equation}
\label{eq:gap-ucb}
\big|\Delta r^\star(\x;y,y')-\widehat{\Delta r}_{t}(\x;y,y')\big|\ \le\ b_t(\x,y,y').
\end{equation}
Since $\sigma(\cdot)$ is monotone, Eq.~\eqref{eq:gap-ucb} implies an upper bound on the preference probability:
\begin{equation*}
\Prob^\star(y\succ y'\mid \x)\ \le\ \sigma\big(\widehat{\Delta r}_{t}(\x;y,y')+b_t(\x,y,y')\big).
\end{equation*}

Under the DPO reward-policy correspondence, the empirical reward gap can be parameterized by the policy margin, i.e., $\widehat{\Delta r}_{t}(\x;y,y')=\beta m_\pi(\x;y,y')$. We define the data-dependent exploration bonus as
\begin{equation*}
\label{eq:method4_bonus}
G_{\mathrm{DEPO}}(\pi,b_t) = \E_{\substack{\x\sim \rho,\\ y \sim \pi(\cdot\mid \x),\\ y'\sim\pi_t^{\mathrm{sam}}(\cdot\mid \x)}}
\Big[\sigma\big(\beta\,m_\pi(\x;y,y')+b_t(\x,y,y')\big)\Big].
\end{equation*}
We note that the radius $b_t(\x,y,y')$ will be instantiated later using a representation-based confidence bound from historical comparisons.

At round $t$, we update the policy by
\begin{equation}
\label{eq:our_obj}
\pi_{t+1} = \argmax_{\pi\in\Pi} L_{\rm DPO}(\pi,D_t) + \alpha \, G_{\mathrm{DEPO}}(\pi, b_t),
\end{equation}
where $\alpha>0$ controls the exploration strength.

\subsection{Representation-based Bonus Design}
Here, we construct the data-dependent exploration bonus using a self-normalized confidence radius in the representation space. Thus, the bonus can adapt to the empirical geometry of historical comparisons.

Let $\beta_t^{\mathrm{conf}}>0$ be the confidence width, and let $V_t$ denote the regularized covariance matrix constructed from historical pairwise representations, defined as $V_t= \lambda I + \sum_{s=1}^{t} \psi(\x_s,y_s,y'_s)\psi(\x_s,y_s,y'_s)^\top$, where $\lambda>0$ is the regularization parameter and $I$ is the identity matrix. For notational simplicity, we write $\psi_s=\psi(\x_s,y_s,y'_s)$ for the pairwise representation observed at round $s$. We instantiate $b_t(\x,y,y')$ as the elliptical UCB radius:
\begin{equation}
\label{eq:method4_bt}
b_t(\x,y,y') = \beta_t^{\mathrm{conf}} \sqrt{\psi(\x,y,y')^\top V_t^{-1}\psi(\x,y,y')}.
\end{equation}
This bonus is larger when the direction $\psi(\x,y,y')$ is less well covered by the past comparisons encoded in $V_t$. As a result, it assigns a stronger exploration incentive to under-explored regions.

The theoretical confidence width $\beta_t^{\mathrm{conf}}$ depends on the logistic curvature over the bounded reward-gap class induced by Assumption~\ref{ass:linear_rep}. Specifically, define
\begin{equation}
\label{eq:local_curvature}
\kappa_S = \inf_{\theta\in\Theta_S,\;\|\psi\|_2\le 1}\sigma'\big(\langle\theta,\psi\rangle\big)=\sigma(S)(1-\sigma(S))\ge \frac{1}{4}e^{-S}.
\end{equation}
where $\Theta_S=\{\theta\in\mathbb{R}^d:\|\theta\|_2\le S\}$ and $\sigma'(u)=\sigma(u)(1-\sigma(u))$. This curvature lower bound controls the sensitivity of the Bradley--Terry link over all reward gaps allowed by the linear realizability assumption. The following lemma justifies the theoretical choice of $\beta_t^{\mathrm{conf}}$ by showing that the covariance-based radius in Eq.~\eqref{eq:method4_bt} satisfies the reward-gap confidence bound in Eq.~\eqref{eq:gap-ucb}.

\begin{myLemma}
\label{lem:gap-error-by-radius}
Under Assumptions~\ref{ass:BT_model} and~\ref{ass:linear_rep}, for any $\delta\in(0,1)$, define $\beta_t^{\mathrm{conf}} = \frac{c_0}{\kappa_S}\left(\sqrt{\lambda}\,S+\sqrt{2\log\left(\frac{\det(V_t)^{1/2}}{\det(\lambda I)^{1/2}\,\delta}\right)}\right)$ for a universal constant $c_0>0$. Then, with probability at least $1-\delta$, for all $t\in\{1,\ldots,T\}$ and all $(\x,y,y')\in\X\times\Y\times\Y$, we have
\begin{equation*}
\big|\Delta r^\star(\x;y,y')-\widehat{\Delta r}_{t}(\x;y,y')\big| \le \beta^{\mathrm{conf}}_{t}\sqrt{\psi(\x,y,y')^\top V_{t}^{-1}\psi(\x,y,y')}.
\end{equation*}
\end{myLemma}

\begin{myRemark}
The confidence guarantee in Lemma~\ref{lem:gap-error-by-radius} relies on the linear reward-gap model in Assumption~\ref{ass:linear_rep}. In the proof, this model is used to construct an auxiliary logistic estimator and to show that the covariance-based radius in Eq.~\eqref{eq:method4_bt} upper-bounds the reward-gap estimation error in Eq.~\eqref{eq:gap-ucb}. The practical bonus retains the same covariance term $\|\psi(x,y,y')\|_{V_t^{-1}}$ as its data-dependent uncertainty measure.
\end{myRemark}

The key observation is that the radius adapts to realized trajectory through $V_t$, while the curvature cost is controlled by reward scale $S$.

To adapt elliptical UCB radius to LLMs, we use representations extracted from the LLM itself as the feature vectors. Specifically, we use a pair-wise feature $\psi(\x,y,y')$ to model the latent reward gap in a linear form. Following~\cite{tuyls2025representation}, for each prompt-response pair $(\x,y)$, we construct the corresponding representation using a frozen policy network's last-layer hidden states. Specifically, we mean-pool the last-layer hidden states over the output tokens of the generated response and apply a projection to a fixed dimension, i.e.,
\begin{equation}
\label{eq:repexp_repr_simple}
\phi(\x,y) = \mathrm{Proj}\left(\frac{1}{|y|}\sum_{t=1}^{|y|} h_{\theta}(\x,y_{1:t})\right)\in\R^{d},
\end{equation}
where $h_{\theta}(\x,y_{1:t})\in\mathbb{R}^{d'}$ is the last-layer hidden state at generation step $t$ and $|y|$ is the response length. We use a fixed sparse random linear projection $\mathrm{Proj}(v)=Pv$ with $P\in\mathbb{R}^{d\times d'}$, where $d$ is the projected dimension and $P$ is constant throughout training~\cite{li2006very}.

\subsection{Practical Implementation}
Since computing $\beta_t^{\mathrm{conf}}$ from its theoretical expression can be cumbersome in practice, we instead estimate it from historical data. We define the empirical radius as
\begin{equation*}
\label{eq:typical_radius}
\bar r_t = \operatorname{median}_{s\in\mathcal B_t} \sqrt{\psi_s^\top V_{t-1}^{-1}\psi_s},
\end{equation*}
where $\operatorname{median}_{s\in\mathcal B_t} a_s$ denotes the median of the set $\{a_s:s\in\mathcal B_t\}$ and $\mathcal B_t$ is a buffer that stores historical comparisons. We then use $\gamma_t$ to approximate $\beta_t^{\mathrm{conf}}$, which is defined as
\begin{equation*}
\label{eq:beta_conf_simple}
\gamma_t = \frac{c_b}{\bar r_t+\varepsilon},
\end{equation*}
with constants $c_b>0$ and $\varepsilon>0$. This choice keeps the typical magnitude of the bonus $b_t(\x,y,y')=\beta_t^{\mathrm{conf}}\sqrt{\psi(\x,y,y')^\top V_{t-1}^{-1}\psi(\x,y,y')}$ stable across rounds.

We note that computing the data-dependent bonus $G_{\mathrm{DEPO}}(\pi,b_t)$ exactly requires on-policy sampling. While feasible, such sampling is time-consuming. To avoid this cost, following common practice~\cite{cen2024value,xie2024exploratory,chen2025avoiding}, we optimize the following objective:
\begin{equation}
\label{eq:our_pruned_obj}
\begin{aligned}
& \pi_{t+1} = \argmax_{\pi\in\Pi} \sum_{s=1}^t \log \sigma\Big(\beta \, m_\pi(\x_s;y_s^{\mathrm{w}},y_s^{\mathrm{l}}) \Big)\\
& + \alpha \sum_{s=1}^{t} \sigma\left(\beta\,\log\frac{\pi(y_s'\mid\x_s)}{\pi_{\rm ref}(y_s'\mid\x_s)} + \gamma_t\,\sqrt{\psi(x_s,y_s,y_s')^\top V_t^{-1}\psi(x_s,y_s,y_s')}\right),
\end{aligned}
\end{equation}
where $y'$ is the response sampled from the sampler policy each round, $\alpha\ge 0$ is the exploration weight.

To avoid repeatedly inverting $V_t$, we update its inverse using the Sherman--Morrison formula~\cite{sherman1950adjustment} when adding a new feature vector $\psi$:
\begin{equation}
\label{eq:sherman_morrison}
V_{t+1}^{-1} = V_t^{-1} - \frac{V_t^{-1}\psi\,\psi^\top V_t^{-1}} {1+\psi^\top V_t^{-1}\psi}.
\end{equation}
We perform covariance-related computations in double precision for numerical stability.

In practice, we design the sampler policy as follows. Fix an interval $H\ge 1$. The sampler is refreshed every $H$ rounds to stabilize sampling within each block. For each block index $k\ge 0$ and all $t\in\{kH+1,\dots,(k+1)H\}\cap\{1, \ldots, T\}$,
\begin{equation}
\label{eq:sampler_refresh}
\pi_t^{\mathrm{sam}} = \pi_{kH+1}.
\end{equation}

We summarize the proposed \emph{data-dependent exploration for preference optimization}~{(DEPO)} algorithm in Algorithm~\ref{alg:method}.

\begin{algorithm}[t]
\caption{Data-dependent Exploration for Preference Optimization (DEPO)}\label{alg:method}
\begin{algorithmic}[1]
\STATE \textbf{Input:} Policy $\piRef$, initial sampler $\pi_1^{\mathrm{sam}}$, horizon $T$, $\beta$, $\alpha$, regularization $\lambda$
\STATE Initialize $\D_0=\emptyset$, $V_0=\lambda I$, choose any $\pi_1\in\Pi$.
\FOR{$t=1$ to $T$}
    \STATE Sample $x_t\sim\rho$, $y_t\sim\pi_t(\cdot\mid x_t)$, $y'_t\sim\pi_t^{\mathrm{sam}}(\cdot\mid x_t)$
    \STATE Query oracle to obtain $(y_t^{\mathrm{w}},y_t^{\mathrm{l}})$ from $(y_t,y'_t)$
    \STATE Update $\D_t\gets \D_{t-1}\cup\{(x_t,y_t^{\mathrm{w}},y_t^{\mathrm{l}})\}$
    \STATE Set $\psi_t \gets \psi(x_t,y_t^{\mathrm{w}},y_t^{\mathrm{l}})$, update $V_t^{-1}$ via Eq.~\eqref{eq:sherman_morrison}
    \STATE Set $b_t(x,y,y')=\gamma_t\sqrt{\psi(x,y,y')^\top V_t^{-1}\psi(x,y,y')}$
    \STATE Update $\pi_{t+1}$ via Eq.~\eqref{eq:our_pruned_obj}
    \STATE Update $\pi_t^{\mathrm{sam}}$ via Eq.~\eqref{eq:sampler_refresh}
\ENDFOR
\STATE \textbf{Output:} $\pi_{T+1}$
\end{algorithmic}
\end{algorithm}

\subsection{Instance-Dependent Regret Guarantees}
Here we present our theoretical result, showing that the proposed bonus yields a preference regret whose exploration term is \emph{data-dependent}.

We first introduce a task-hardness measure based on the geometric properties of the representation space.

\begin{myDef}[Representation diversity]
\label{def:feature_diversity}
Let $\mathcal F_{t-1}$ be the history before sampling at round $t$. An RLHF task with prompt distribution $\rho$, representation $\psi$, and policy class $\Pi$ is called conditionally $\gamma$-diverse if, for every round $t$ and every policy pair $(\pi_t,\pi_t^{\mathrm{sam}})$ reachable by the theoretical DEPO update,
\begin{equation}
\label{eq:diversity}
\E\!\left[\psi_t\psi_t^\top\mid\mathcal F_{t-1}\right] \succeq \gamma\, I,
\end{equation}
where $\psi_t=\psi(\x_t,y_t,y_t')$, $\x_t\sim\rho$, $y_t\sim\pi_t(\cdot\mid \x_t)$, $y_t'\sim\pi_t^{\mathrm{sam}}(\cdot\mid \x_t)$, $I\in\mathbb{R}^{d\times d}$ is the identity matrix, and $\gamma>0$.
\end{myDef}

\begin{myRemark}
The $\gamma$-diversity condition requires that the pairwise features induced by online sampling provide sufficient conditional coverage of the representation space. This is plausible when the prompt distribution is diverse and the representation captures meaningful semantic distinctions. Here, $\gamma$ quantifies the extent of coverage: larger $\gamma$ corresponds to richer variation and easier exploration, while smaller $\gamma$ indicates weaker feature diversity.
\end{myRemark}

\begin{myThm}
\label{thm:main_instance_dependent}
Under Assumptions~\ref{ass:BT_model}, \ref{ass:linear_rep}, and~\ref{ass:finite}, suppose the task is conditionally $\gamma$-diverse. Running DEPO update with the confidence radius in Eq.~\eqref{eq:method4_bt}, with probability at least $1-3\delta$, choosing $\alpha=\sqrt{T}$ yields
\begin{equation*}
\mathrm{R}_{\mathrm{pref}}(T)=\tilde O\left(\sqrt{T}\log\frac{|\mathcal R|}{\delta}+e^{O(S)}\sqrt d\left(\sqrt{\frac{T}{\gamma}}+\frac{1}{\gamma^2}\right)\right).
\end{equation*}
Without the diversity condition, we have
\begin{equation*}
\mathrm{R}_{\mathrm{pref}}(T) = \tilde O\left(\sqrt{T}\log\frac{|\mathcal{R}|}{\delta} + e^{O(S)}\,d\sqrt{T}\right).
\end{equation*}
\end{myThm}

\begin{myRemark}
\label{rem:comparison}
The first term $\tilde O(\sqrt T\log(|\mathcal R|/\delta))$ comes from the finite-class likelihood comparison and is independent of the covariance geometry. Our data-dependent analysis improves the second term, namely the cumulative exploration cost. Intuitively, when the collected comparisons provide sufficient representation coverage, the empirical covariance $V_t$ grows in many directions, making the uncertainty radius $\|\psi\|_{V_t^{-1}}$ shrink faster for well-covered directions. Under conditional $\gamma$-diversity, this yields the cumulative exploration term $\tilde O(e^{O(S)}\sqrt{dT/\gamma})$, compared with the worst-case elliptical-potential rate $\tilde O(e^{O(S)}d\sqrt T)$. 

This improvement is orthogonal to prior reward-scale improvements: existing methods such as \cite{xie2024exploratory} and \cite{cen2024value} focus on reward-regret guarantees, while \cite{chen2025avoiding} studies preference regret and improves reward-scale dependence. In contrast, our bound shows that the exploration cost can adapt to the representation coverage of the collected trajectory. The bonus is data-dependent in the sense that its pair-specific uncertainty component $\|\psi(x,y,y')\|_{V_t^{-1}}$ is larger for directions poorly covered by the empirical covariance and smaller for directions already supported by historical comparisons.
\end{myRemark}

%% file: section/experiments.tex
\section{Experiments}
\label{sec:experiments}

In this section, we first describe the experimental setup, then compare DEPO with three strong baselines on eight benchmarks, and finally present ablation studies and a code generation case study.

\subsection{Experimental Setup}
\label{subsec:exp_setup}

\textbf{Benchmarks.}
Following the practice of~\cite{xie2024exploratory} and~\cite{chen2025avoiding}, we evaluate on six benchmark datasets: MMLU~\cite{hendrycksmeasuring}, GPQA~\cite{rein2024gpqa}, TruthfulQA~\cite{lin2022truthfulqa}, GSM8k~\cite{cobbe2021training}, AlpacaEval (AE2)~\cite{dubois2024length}, and MT-bench (MT)~\cite{bai2024mt}. These public benchmarks share a common characteristic: the preference model used for training differs from the one used for evaluation. Such a train-test shift in the preference model is uncommon in machine learning evaluations. Following~\cite{chen2025avoiding}, we further consider two representative settings simulated using the iterative DPO dataset~\cite{dong2024rlhf}.

\emph{1. Domain-specific alignment (referred to as IID).}
The goal is to adapt an LLM to a narrow domain, where prompts and desired behaviors are relatively concentrated. This setting stresses exploration efficiency: the learner should quickly find and refine high-value behaviors without overfitting to early noisy feedback. We evaluate on an IID held-out prompt set drawn from the same distribution as training, using the same preference model as the oracle.

\emph{2. Generalist alignment (referred to as Alpaca).}
The goal is to train a general assistant that works across diverse topics and intents. This setting emphasizes robustness and coverage: exploration should discover under-covered yet important behaviors without sacrificing performance on common prompts. We evaluate on AlpacaEval 2.0 prompts, while still using the same preference model as the oracle.

\begin{table*}[!t]
\centering
\resizebox{0.98\linewidth}{!}{
\begin{tabular}{@{}lcccccccccc@{}}
\toprule
\textbf{Model} &
\textbf{MMLU} &
\textbf{GPQA} &
\textbf{TruthfulQA} &
\textbf{GSM8k} &
\textbf{AE2} &
\textbf{MT} &
\multicolumn{2}{c}{\textbf{IID}} &
\multicolumn{2}{c}{\textbf{Alpaca}} \\
\cmidrule(lr){8-9} \cmidrule(lr){10-11}
 &  &  &  &  &  &  & \textbf{WR} & \textbf{AvgR} & \textbf{WR} & \textbf{AvgR} \\
\midrule
Llama-3-8B-SFT & 62.61 & 32.40 & 53.50 & 71.77 & 10.22  & 7.66 & -- & -5.54 & -- & -7.92 \\
\midrule
DPO-iter1 & 62.87 $\pm$ 0.46 & \textbf{32.84 $\pm$ 0.39} & 56.24 $\pm$ 0.71 & 76.65 $\pm$ 0.52 & -- & -- & 62.4 $\pm$ 0.63 & -4.50 $\pm$ 0.24 & 78.1 $\pm$ 0.58 & -6.02 $\pm$ 0.16 \\
DPO-iter2 & 62.92 $\pm$ 0.54 & 30.76 $\pm$ 0.67 & 57.66 $\pm$ 0.48 & 77.80 $\pm$ 0.36 & -- & -- & 66.6 $\pm$ 0.72 & -3.59 $\pm$ 0.35 & 87.1 $\pm$ 0.41 & -3.34 $\pm$ 0.29 \\
DPO-iter3 & 63.08 $\pm$ 0.38 & 31.24 $\pm$ 0.74 & 59.61 $\pm$ 0.45 & 77.38 $\pm$ 0.62 & 36.14 $\pm$ 0.57 & 8.29 $\pm$ 0.34 & 72.5 $\pm$ 0.66 & -2.36 $\pm$ 0.13 & 91.2 $\pm$ 0.47 & -0.02 $\pm$ 0.19 \\
\midrule
XPO-iter1 & 62.66 $\pm$ 0.51 & 32.47 $\pm$ 0.43 & 56.19 $\pm$ 0.65 & 76.42 $\pm$ 0.48 & -- & -- & 62.6 $\pm$ 0.70 & -4.40 $\pm$ 0.26 & 78.3 $\pm$ 0.61 & -5.79 $\pm$ 0.15 \\
XPO-iter2 & 63.13 $\pm$ 0.36 & 31.25 $\pm$ 0.59 & 58.61 $\pm$ 0.44 & 77.30 $\pm$ 0.73 & -- & -- & 67.3 $\pm$ 0.50 & -3.28 $\pm$ 0.37 & 88.0 $\pm$ 0.52 & -2.60 $\pm$ 0.24 \\
XPO-iter3 & 63.12 $\pm$ 0.62 & 31.01 $\pm$ 0.55 & 59.34 $\pm$ 0.39 & \textbf{78.41 $\pm$ 0.46} & 38.23 $\pm$ 0.68 & 8.21 $\pm$ 0.41 & 72.9 $\pm$ 0.57 & -2.11 $\pm$ 0.34 & 91.6 $\pm$ 0.35 & 0.64 $\pm$ 0.29 \\
\midrule
POPO-iter1 & 62.85 $\pm$ 0.44 & 32.38 $\pm$ 0.72 & 56.11 $\pm$ 0.50 & 76.88 $\pm$ 0.37 & -- & -- & 62.5 $\pm$ 0.59 & -4.32 $\pm$ 0.36 & 80.0 $\pm$ 0.48 & -5.68 $\pm$ 0.25 \\
POPO-iter2 & 62.95 $\pm$ 0.53 & 31.72 $\pm$ 0.40 & 57.81 $\pm$ 0.69 & 77.39 $\pm$ 0.58 & -- & -- & 68.2 $\pm$ 0.45 & -3.15 $\pm$ 0.22 & 89.1 $\pm$ 0.73 & -2.45 $\pm$ 0.10 \\
POPO-iter3 & 63.18 $\pm$ 0.41 & 31.94 $\pm$ 0.63 & 59.07 $\pm$ 0.56 & 77.63 $\pm$ 0.49 & 40.13 $\pm$ 0.71 & 8.38 $\pm$ 0.38 & 73.2 $\pm$ 0.54 & -2.02 $\pm$ 0.16 & 92.4 $\pm$ 0.67 & 0.60 $\pm$ 0.23 \\
\midrule
DEPO-iter1 & 62.88 $\pm$ 0.47 & 32.83 $\pm$ 0.52 & 56.14 $\pm$ 0.64 & 76.55 $\pm$ 0.39 & -- & -- & 62.6 $\pm$ 0.58 & -4.30 $\pm$ 0.31 & 80.2 $\pm$ 0.50 & -5.61 $\pm$ 0.22 \\
DEPO-iter2 & 63.12 $\pm$ 0.35 & 32.19 $\pm$ 0.68 & 58.26 $\pm$ 0.43 & 77.49 $\pm$ 0.55 & -- & -- & 68.9 $\pm$ 0.49 & -3.01 $\pm$ 0.27 & 89.5 $\pm$ 0.44 & -2.18 $\pm$ 0.19 \\
DEPO-iter3 & \textbf{63.23 $\pm$ 0.40} & 32.50 $\pm$ 0.46 & \textbf{59.82 $\pm$ 0.51} & 78.39 $\pm$ 0.42 & \textbf{41.56 $\pm$ 0.60} & \textbf{8.45 $\pm$ 0.36} & \textbf{74.1 $\pm$ 0.53} & \textbf{-1.94 $\pm$ 0.28} & \textbf{92.8 $\pm$ 0.39} & \textbf{0.66 $\pm$ 0.21} \\
\bottomrule
\end{tabular}
}
\vspace{2mm}
\caption{\textbf{Main comparison across three online RLHF iterations.} DEPO achieves the strongest overall performance at iter3, with consistent gains on AE/MT and improved win rates on both IID and Alpaca.}
\label{tab:opo_main}
\end{table*}

\textbf{Contenders.} We compare the proposed method with DPO~\cite{rafailov2023direct}, XPO~\cite{xie2024exploratory}, and POPO~\cite{chen2025avoiding}. For a fair comparison, all methods use the same base model, preference oracle, and sampler design as POPO~\cite{chen2025avoiding}, and are trained for the same number of iterations.

\textbf{Implementation details.} We follow the three-iteration online RLHF setting considered in~\cite{xie2024exploratory,cen2024value,chen2025avoiding}. We use \texttt{Llama-3-8B-Flow-SFT} as the base model, \texttt{RLHFlow-ultrafeedback} dataset as the training prompt sets, and \texttt{GRM-Llama3-8B-rewardmodel-ft} as the training preference model. Across all three iterations, we keep the same base model initialization.

\textbf{Representation construction.}
We follow RepExp~\cite{tuyls2025representation} to construct a vector representation $\phi(\x,y)$ for each prompt-response pair from the LM hidden states. We extract representations using the same reference model \texttt{Llama-3-8B-Flow-SFT} used throughout online RLHF.

Specifically, given a response $y=y_{1:T}$, we record the last-layer hidden state at each generation step and mean-pool across output tokens, setting $\phi(\x,y)=\frac{1}{T}\sum_{t=1}^{T} h_{\text{ref}}(\x,y_{1:t})$. We use mean pooling rather than using only the final-token or penultimate-token representation. For efficiency, we also project $\phi(x,y)$ to 512 dimensions via a sparse random projection. Within each prompt, we mean-center the projected representations over the sampled responses before computing uncertainty-related quantities, and we perform covariance-related operations in double precision for numerical stability.

\subsection{Main Results}
\label{subsec:main_results}

Table~\ref{tab:opo_main} reports the comparison results over three online RLHF iterations. After three iterations, the proposed DEPO algorithm achieves the best final performance on most of the datasets, with consistent improvements on the domain-specific and generalist alignment tasks. In particular, DEPO attains higher win rates and average rewards on both IID and Alpaca compared to DPO and prior exploration baselines at iteration 3. These results are consistent with the interpretation that DEPO improves online RLHF not by simply collecting more preference data, but by collecting more informative comparisons that are better aligned with optimization. This trend matches the intended role of the data-dependent bonus, which steers exploration toward comparisons that complement the historical coverage and thus translate into improved alignment outcomes.

On the GPQA dataset, we observe a non-monotonic performance trend across iterations: performance drops at iteration 2 and recovers at iteration 3 for all methods. This suggests that exploration can introduce short-term drawbacks when it pushes the learner into uncertain regions, even though it may ultimately encourage collecting under-covered yet informative behaviors without overly biasing training toward ``easy'' comparisons. Against this backdrop, DEPO exhibits a smaller degradation at iteration 2 and a stronger recovery at iteration 3 compared with DPO and earlier exploration baselines. This pattern suggests that DEPO's data-dependent exploration bonus leverages historical data to steer exploration in a more effective direction.

Overall, these results support that introducing a data-dependent exploration bonus improves both online alignment efficiency and generalization across tasks.

\subsection{Ablation Studies}
\label{subsec:ablation}

We then run ablation studies to better understand the behavior of the DEPO algorithm. We focus on the strength of the data-dependent exploration bonus, the sampler used to generate preference pairs, and the RLHF hyperparameter $\beta$.

\textbf{Representation Coverage.} Since the proposed bonus depends on the geometry of the pairwise representation space, we analyze the empirical covariance spectrum of the collected pairwise features. Let $\Sigma_t=\sum_{s=1}^t\psi_s\psi_s^\top$. We report the effective dimension $d_{\mathrm{eff}}(\Sigma_t)/d=(\mathrm{tr}(\Sigma_t))^2/d\,\mathrm{tr}(\Sigma_t^2)$ and the top-eigenvalue energy ratio, namely, $\rho_k(\Sigma_t)=\sum_{i=1}^k\lambda_i(\Sigma_t)/\sum_{i=1}^d\lambda_i(\Sigma_t)$. As shown in Table~\ref{tab:coverage_diagnostic}, DEPO achieves a larger $d_{\mathrm{eff}}/d$ and smaller $\rho_{10}$ and $\rho_{50}$ than the baselines. A higher effective dimension, together with lower top-eigenvalue energy ratios, indicates that the collected comparisons are less concentrated and span a broader range of directions in the representation space.

\begin{table}[t]
\centering
\caption{\textbf{Representation coverage on the collected pairwise features.} Larger $d_{\mathrm{eff}}$ and smaller eigenvalue concentration indicate broader coverage of the representation space.}
\label{tab:coverage_diagnostic}
\begin{tabular}{lcccc}
\toprule
Method & $d_{\mathrm{eff}}/d\uparrow$ & $\rho_{10}\downarrow$ & $\rho_{50}\downarrow$ \\
\midrule
DPO & 0.73 & 0.69 & 0.82 \\
XPO & 0.75 & 0.62 & 0.77 \\
POPO & 0.77 & 0.62 & 0.78 \\
DEPO & 0.81 & 0.59 & 0.72 \\
\bottomrule
\end{tabular}
\end{table}

\textbf{Effect of exploration strength $c_b$.}
We further ablate the proxy scale $c_b$, which controls the strength of the data-dependent exploration bonus through the practical approximation of the confidence width. We compare three representative values $c_b\in\{1\mathrm{e}{-}1,\,2\mathrm{e}{-}2,\,1\mathrm{e}{-}2\}$ and report the training dynamics using the effective comparison ratio in Figure~\ref{fig:effratio_da_ga}. Across both domain-specific alignment and generalist alignment, all settings exhibit consistent improvement from iteration 1 to 3, indicating that the overall online optimization remains stable under different bonus scales. Moreover, $c_b=2\mathrm{e}{-}2$ shows the most favorable trend, achieving a consistently higher effective comparison ratio, especially at later iterations. This suggests that a moderate bonus scale better balances encouraging informative exploration and avoiding overly aggressive perturbations. Unless otherwise stated, we use $c_b=2\mathrm{e}{-}2$ in all experiments.

\textbf{Choice of sampler policy.}
In the online RLHF procedure, we generate responses for each prompt using the current policy together with a sampler policy. We ablate this design choice because the sampler directly controls the difficulty of the comparisons and the effectiveness of learning.

Specifically, we consider two sampler configurations: (i) an \emph{on-policy} sampler $(\pi_t,\pi_t)$, where both responses are sampled from the current policy at iteration $t$; and (ii) a \emph{mixed} sampler $(\pi_{t-1},\pi_{\rm ref})$, where one response is sampled from the previous-iteration policy and the other from the initial reference policy $\pi_0$. The mixed sampler can help preserve response diversity by anchoring one candidate to the initial model, but it may also generate overly easy comparisons, which contribute less to refining the decision boundary between competitive responses.

\begin{figure}[t]
\centering
\begin{minipage}[t]{0.48\textwidth}
  \centering
  \includegraphics[width=\linewidth]{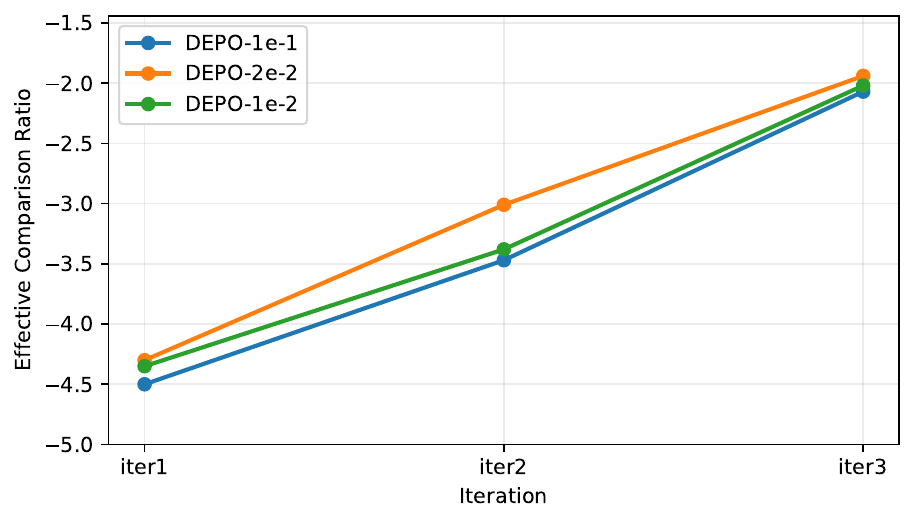}

  \vspace{1mm}
  \small \textbf{(a) Domain-specific alignment task.}
  \label{fig:effratio_da}
\end{minipage}\hfill
\begin{minipage}[t]{0.48\textwidth}
  \centering
  \includegraphics[width=\linewidth]{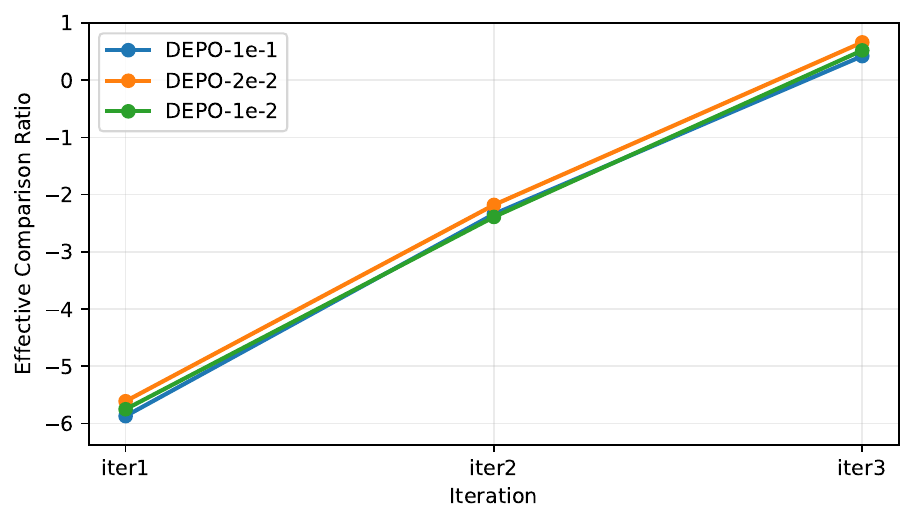}

  \vspace{1mm}
  \small \textbf{(b) Generalist alignment task.}
  \label{fig:effratio_ga}
\end{minipage}
\caption{\textbf{Ablation on the bonus scale $c_b$.} A moderate choice $c_b=2\mathrm{e}{-}2$ yields the strongest overall trend.}
\label{fig:effratio_da_ga}
\end{figure}

\begin{table}[!t]
\centering
\begin{tabular}{lcc}
\toprule
\textbf{Model} & \textbf{WR} & \textbf{AvgR} \\
\midrule
$(\pi_t,\pi_t)$-iter2 & 88.6 & -2.84 \\
$(\pi_t,\pi_t)$-iter3 & 92.5 & 0.21 \\
\midrule
$(\pi_{t-1},\piRef)$-iter2 & 87.8 & -4.09 \\
$(\pi_{t-1},\piRef)$-iter3 & 91.7 & -2.63 \\
\bottomrule
\end{tabular}
\vspace{2mm}
\caption{\textbf{Average reward and win rate comparison under different sampler configurations.} The on-policy sampler $(\pi_t,\pi_t)$ yields higher WR and AvgR at later iterations, suggesting more informative comparisons.}
\label{tab:sampler_choice}
\end{table}

We report the comparison results in Table~\ref{tab:sampler_choice}. Overall, both sampler choices lead to iterative improvements, indicating that DEPO is fairly robust to the sampler configuration. However, the on-policy sampler $(\pi_t,\pi_t)$ achieves consistently higher win rate and average reward at later iterations, which suggests that pairing two candidates from the current policy yields harder and more informative comparisons, thereby accelerating improvement. Unless otherwise stated, we use $(\pi_t,\pi_t)$ as the default sampler in all experiments.

\textbf{Effect of the exploration weight $\alpha$.}
We further examine the sensitivity of DEPO to the exploration weight $\alpha$, which controls the strength of the data-dependent bonus. Table~\ref{tab:alpha_ablation} reports the final-iteration performance under different values of $\alpha$. The results show that a moderate exploration weight gives the best overall performance. Increasing $\alpha$ from zero improves performance by encouraging exploration toward under-covered regions, while an overly large value can overemphasize exploration and slightly degrade the final reward. These results support the default choice of $\alpha$ used in the main experiments.

\textbf{Effect of the hyperparameter $\beta$.}
We also conduct an ablation study to examine the effect of the KL coefficient $\beta$. Following the practical design in~\cite{chen2025avoiding}, we co-tune the exploration weight $\alpha$ with $\beta$ to keep $\alpha\beta$ constant, so that the gradient scale of the exploration term remains comparable across settings. We evaluate DEPO with three representative choices of $\beta$ and report the results in Table~\ref{tab:beta_ablation}. Across MMLU, GSM8k, and AlpacaEval, the method improves steadily from iter1 to iter3 under all $\beta$ values, indicating that the iterative procedure is robust to moderate changes in the KL strength. However, performance is not monotonic in $\beta$: the intermediate setting $\beta=3\mathrm{e}{-}2$ consistently gives the strongest iter3 results, while larger $\beta$ tends to constrain updates more tightly and smaller $\beta$ can lead to less stable gains. Overall, a moderate $\beta$ provides the best trade-off between regularization and optimization progress.

\begin{table}[t]
\centering
\caption{\textbf{Ablation study on parameter $\alpha$}}
\label{tab:alpha_ablation}
\begin{tabular}{lcccc}
\toprule
$\alpha$ & $\sqrt{T}/10$ & $\sqrt{T}/2$ & $\sqrt{T}$ & $2\sqrt{T}$ \\
\midrule
IID WR $\uparrow$ & 72.2 & 73.5 & 74.1 & 72.4 \\
IID AvgR $\uparrow$ & -3.89 & -2.33 & -1.94 & -3.54 \\
Alpaca WR $\uparrow$ & 90.2 & 91.4 & 92.8 & 89.6 \\
Alpaca AvgR $\uparrow$ & -1.06 & -0.57 & 0.66 & -1.33 \\
\bottomrule
\end{tabular}
\end{table}

\begin{table}[!t]
\centering
\begin{tabular}{l l c c c c}
\toprule
\textbf{$\beta$} & \textbf{Model} & \textbf{MMLU} & \textbf{GSM8k} & \multicolumn{2}{c}{\textbf{Alpaca Data}} \\
\cmidrule(lr){5-6}
 &  &  &  & \textbf{WR} & \textbf{AvgR} \\
\midrule
\multirow{3}{*}{1e-1}
 & DEPO-iter1 & 62.45 & 75.24 & 80.5 & -5.77 \\
 & DEPO-iter2 & 62.97 & 76.32 & 89.3 & -2.56 \\
 & DEPO-iter3 & 63.06 & 77.28 & 92.3 & 0.59 \\
\midrule
\multirow{3}{*}{3e-2}
 & DEPO-iter1 & 62.88 & 76.55 & 80.2 & -5.61 \\
 & DEPO-iter2 & 63.12 & 77.49 & 89.5 & -2.18 \\
 & DEPO-iter3 & 63.23 & 78.39 & 92.8 & 0.66 \\
\midrule
\multirow{3}{*}{1e-2}
 & DEPO-iter1 & 62.68 & 75.76 & 78.9 & -5.82 \\
 & DEPO-iter2 & 63.02 & 76.59 & 89.3 & -2.68 \\
 & DEPO-iter3 & 63.09 & 77.80 & 92.1 & 0.50 \\
\bottomrule
\end{tabular}
\vspace{2mm}
\caption{\textbf{Ablation study on parameter $\beta$.}}
\label{tab:beta_ablation}
\end{table}

\subsection{A Real-world Application}
\label{subsec:realworld_lcb}

We further evaluate DEPO on the code generation task using the LiveCodeBench dataset~\cite{jainlivecodebench}. In this setting, each prompt $\x$ is a programming problem and each response $y$ is a complete program generated by the model. At round $t$, we sample a problem $\x_t$, generate two candidate programs $y_t$ and $y_t'$, and execute both programs on the corresponding test cases to obtain preference feedback. We convert the execution outcomes into pairwise preference labels by defining $y_t \succ y_t'$ whenever $y_t$ passes more test cases than $y_t'$. If both programs pass the same number of test cases, we break ties by preferring the one with fewer compilation or runtime errors. The resulting preference pair is then added to the online training buffer for policy update. 

We follow the official LiveCodeBench protocol and report Pass@1 separately on the Easy, Medium, and Hard splits. For online training, coding problems are sampled from the training portion of LiveCodeBench, and evaluation is performed on a disjoint held-out test set. We report the results in Table~\ref{tab:lcb_main}. Overall, DEPO consistently improves upon the base LLM across the three difficulty splits of LiveCodeBench and achieves the best performance among the compared methods. These results indicate that the proposed data-dependent exploration mechanism remains effective in code generation, where preference feedback is derived directly from executable test outcomes. This also suggests that the benefit of DEPO is not limited to general tasks, but extends to real-world tasks with verifiable correctness.

\begin{table}[!t]
\centering
\begin{tabular}{lccc}
\toprule
\textbf{Model} & \textbf{Easy} & \textbf{Medium} & \textbf{Hard} \\
\midrule
Qwen2.5-Coder-7B-Instruct & 56.1 & 3.8 & 6.9 \\
+ DPO-iter3 & 57.6 & 14.5 & 8.7 \\
+ XPO-iter3 & 58.4 & 15.2 & 9.4 \\
+ POPO-iter3 & 58.2 & 15.0 & 9.2 \\
+ DEPO-iter3 & 62.5 & 18.8 & 9.7 \\
\bottomrule
\end{tabular}
\vspace{2mm}
\caption{\textbf{Results on LiveCodeBench code generation.} Comparison of online RLHF methods with execution-based preference feedback.}
\label{tab:lcb_main}
\end{table}

%% file: section/related.tex
%!TEX root = ../FCS/template.tex

\section{Related Work}
\label{sec:related}

In this section, we review the relevant literature and related techniques.

\paragraph{Reinforcement Learning from Human Feedback} A closely related RLHF framework to our work was introduced by~\cite{christiano2017deep}, which learns from preference comparisons and later became a core technique for aligning LLMs with human preferences. The widely adopted RLHF pipeline~\cite{ouyang2022training,touvron2023llama} typically follows two stages: (i) training a reward model from preference data, and (ii) refining the LLM policy against the learned reward using policy-optimization methods such as \emph{proximal policy optimization}~{(PPO)}~\cite{schulman2017proximal}. Motivated by the complexity and sensitivity of this two-stage procedure, a fruitful line of work has developed \emph{direct preference optimization}~{(DPO)}~\cite{rafailov2023direct} and related alternatives~\cite{azar2024general,zhao2023slic}, which bypass explicit reward modeling and instead learn the policy by directly optimizing a preference-based objective on the comparison dataset. Moreover, the evaluation of alignment quality has received increasing attention, as gains in safety or preference satisfaction may be accompanied by losses in utility and other important capabilities, calling for a more comprehensive and trustworthy evaluation framework~\cite{zhao2024gains,zhang2025trustworthy}.

\paragraph{Active Exploration in Online RLHF} Recently, online RLHF has attracted substantial attention. Early approaches rely primarily on passive exploration induced by the stochasticity of the LLM policy, without explicit mechanisms to promote diverse sampling or prioritize informative, high-value comparisons~\cite{dong2024rlhf,xiongbuilding}. As the importance of active exploration in RLHF has been increasingly recognized~\cite{dwaracherla2024efficient}, a growing body of work has developed exploration-aware algorithms. A notable result in this direction is online RLHF using the one-pass mirror-descent estimator~\cite{NeurIPS'25:rlhf}, which proposes a provably efficient active querying method but focuses primarily on the prompt space. Another line of work instead encourages exploration in the response space by promoting the generation of novel responses~\cite{cen2024value,zhang2024self,xie2024exploratory,chen2025avoiding}. Most of these approaches augment a DPO-style objective with an exploration bonus to encourage broader coverage during online data collection. While effective in general, they typically employ uniform exploration bonuses across tasks, whereas our method leverages historical data to construct a data-dependent bonus and admits a data-adaptive analysis.

From a technical perspective, our bonus design is related to confidence based exploration methods in bandits and reinforcement learning, especially optimism-style methods that leverage uncertainty estimates to balance exploration and exploitation~\cite{abbasi2011improved,li2010contextual,zhanggeneralized,NeurIPS'25:rlhf}. Related efforts also appear in adaptive allocation frameworks for resource-constrained learning~\cite{wang2024learning}, where they designed an optimal allocation strategy that improves the efficiency in response to evolving uncertainty and learning progress. Although developed in different settings and not directly concerned with online preference optimization for LLM alignment, these works provide intuition for incorporating optimism style exploration into our framework.

%% file: section/appendix.tex
\appendix

\section{Appendix}

In this section, we first present the necessary technical lemmas and then provide the detailed proof of our main theorem.

\subsection{Technical Lemmas}
\label{app:lemma}

In this subsection, we provide the necessary technical lemmas.

\begin{myLemma}[Self-Normalized Concentration]
\label{lem:selfnorm}
Let $\widehat\theta_t$ be the constrained regularized logistic MLE over $\Theta_S$ based on the oriented observations $\{(\psi_s,z_s)\}_{s=1}^t$. Under Assumptions~\ref{ass:BT_model} and~\ref{ass:linear_rep}, there exists a universal constant $c_0>0$ such that, with probability at least $1-\delta$, for all $t\ge 0$,
\begin{equation*}
\|\hat{\theta}_t-\theta^*\|_{V_t} \le \beta_t^{\mathrm{conf}} = \frac{c_0}{\kappa_S}\left( \sqrt{\lambda}S + \sqrt{2\log\left( \frac{\det(V_t)^{1/2}}{\det(\lambda I)^{1/2}\delta} \right)}
\right).
\end{equation*}
\end{myLemma}

\begin{proof}
We denote by $p_s^*=\sigma(\langle\theta^*,\psi_s\rangle)$ and $\varepsilon_s=z_s-p_s^*$. Conditional on the history and the sampled pair, $\psi_s$ is fixed and $\mathbb E[\varepsilon_s]=0$; hence $\varepsilon_s$ is conditionally zero-mean and sub-Gaussian. The self-normalized martingale inequality gives, for $t\ge0$,
\begin{equation*}
\left\|\sum_{s=1}^t\varepsilon_s\psi_s\right\|_{V_t^{-1}} \le \sqrt{2\log \left( \frac{\det(V_t)^{1/2}}{\det(\lambda I)^{1/2}\delta}\right)}.
\end{equation*}
Let $L_t(\theta)=\sum_{s=1}^t\ell_{\log}(z_s,\langle\theta,\psi_s\rangle)+\frac{\lambda}{2}\|\theta\|_2^2$ and $\Delta_t=\hat\theta_t-\theta^*$. Since $\hat\theta_t$ minimizes $L_t$ over the convex set $\Theta_S$ and $\theta^*\in\Theta_S$, the first-order optimality condition for constrained convex optimization gives
\begin{equation*}
\langle \nabla L_t(\hat\theta_t),\theta^*-\hat\theta_t\rangle\ge 0.
\end{equation*}
Since $\nabla_\theta \ell_{\log}(z_s,\langle\theta,\psi_s\rangle)=(\sigma(\langle\theta,\psi_s\rangle)-z_s)\psi_s$, we obtain
\begin{equation*}
\sum_{s=1}^t \left(\sigma(\langle\hat\theta_t,\psi_s\rangle)-z_s\right) \langle\psi_s,\Delta_t\rangle + \lambda\langle\hat\theta_t,\Delta_t\rangle \le 0.
\end{equation*}
Using $z_s=p_s^*+\varepsilon_s$ and $\hat\theta_t=\theta^*+\Delta_t$, this implies
\begin{equation*}
\begin{aligned}
\sum_{s=1}^t
\left(\sigma(\langle\hat\theta_t,\psi_s\rangle)-\sigma(\langle\theta^*,\psi_s\rangle)\right)& \langle\psi_s,\Delta_t\rangle + \lambda\|\Delta_t\|_2^2\\
& \le \left\langle\sum_{s=1}^t\varepsilon_s\psi_s,\Delta_t\right\rangle-\lambda\langle\theta^*,\Delta_t\rangle.
\end{aligned}
\end{equation*}
For each $s$, by the mean-value theorem, there exists $\tilde u_s$ between $\langle\hat\theta_t,\psi_s\rangle$ and $\langle\theta^*,\psi_s\rangle$ such that
\begin{equation*}
\sigma(\langle\hat\theta_t,\psi_s\rangle)-\sigma(\langle\theta^*,\psi_s\rangle) = \sigma'(\tilde u_s)\langle\psi_s,\Delta_t\rangle.
\end{equation*}
Since $\hat\theta_t,\theta^*\in\Theta_S$, $\Theta_S$ is convex, and $\|\psi_s\|_2\le 1$, every intermediate point satisfies $|\tilde u_s|\le S$. Hence $\sigma'(\tilde u_s)\ge \kappa_S$. Therefore,
\begin{equation*}
\kappa_S\sum_{s=1}^t\langle\psi_s,\Delta_t\rangle^2+\lambda\|\Delta_t\|_2^2 \le \left\langle\sum_{s=1}^t\varepsilon_s\psi_s,\Delta_t\right\rangle-\lambda\langle\theta^*,\Delta_t\rangle.
\end{equation*}
Since $\kappa_S\le 1$, the left-hand side is at least $\kappa_S\|\Delta_t\|_{V_t}^2$. By Cauchy--Schwarz,
\begin{equation*}
\left\langle\sum_{s=1}^t\varepsilon_s\psi_s,\Delta_t\right\rangle \le \left\|\sum_{s=1}^t\varepsilon_s\psi_s\right\|_{V_t^{-1}}\|\Delta_t\|_{V_t},
\end{equation*}
and
\begin{equation*}
-\lambda\langle\theta^*,\Delta_t\rangle \le \sqrt{\lambda}\|\theta^*\|_2\|\Delta_t\|_{V_t} \le \sqrt{\lambda}S\|\Delta_t\|_{V_t}.
\end{equation*}
Thus, we have
\begin{equation*}
\kappa_S\|\Delta_t\|_{V_t}^2 \le \left(\left\|\sum_{s=1}^t\varepsilon_s\psi_s\right\|_{V_t^{-1}}+\sqrt{\lambda}S\right)\|\Delta_t\|_{V_t}.
\end{equation*}
If $\|\Delta_t\|_{V_t}=0$, the desired bound is immediate. Otherwise, dividing both sides by $\kappa_S\|\Delta_t\|_{V_t}$ gives
\begin{equation*}
\|\Delta_t\|_{V_t} \le \frac{1}{\kappa_S}\left(\left\|\sum_{s=1}^t\varepsilon_s\psi_s\right\|_{V_t^{-1}}+\sqrt{\lambda}S\right).
\end{equation*}
Combining this with the self-normalized martingale bound proves the claim, up to the universal constant $c_0$.
\end{proof}

\begin{myCor}[Pointwise UCB]
\label{cor:pointwise_ucb}
On the event of Lemma~\ref{lem:selfnorm}, for all $t\ge 0$ and all $(x,y,y')$,
\begin{equation*}
\big|\Delta r^\star(x;y,y')-\widehat{\Delta r}_{t}(x;y,y')\big|\le\beta_t^{\mathrm{conf}}\|\psi(x,y,y')\|_{V_t^{-1}}.
\end{equation*}
\end{myCor}

\begin{proof}
Applying Cauchy--Schwarz in $V_t$ norm gives this bound.
\end{proof}

% \subsection{Covariance Growth on Diverse Tasks}
% \label{sec:instance_analysis}

\begin{myLemma}[Persistent Covariance Growth]
\label{lem:V_growth}
Assume conditional $\gamma$-diversity and $\|\psi_t\|_2\le 1$. Let $t_0=c_1\gamma^{-2}\log(dT/\delta)$ for a universal constant $c_1>0$. With probability at least $1-\delta$, for every $t\ge t_0$,
\begin{equation*}
V_t\succeq \lambda I+\frac{\gamma t}{2}I .
\end{equation*}
\end{myLemma}

\begin{proof}
Set $X_t=\psi_t\psi_t^\top$. Then $0\preceq X_t\preceq I$ and conditional $\gamma$-diversity gives $\mathbb E[X_t\mid\mathcal F_{t-1}]\succeq\gamma I$. Let $Y_s=X_s-\mathbb E[X_s\mid\mathcal F_{s-1}]$, then $\{Y_s\}_{s\ge1}$ is a self-adjoint matrix martingale difference sequence and $\|Y_s\|_{\mathrm{op}}\le 1$. Moreover,
\begin{equation*}
\sum_{s=1}^t X_s = \sum_{s=1}^t \mathbb E[X_s\mid\mathcal F_{s-1}] + \sum_{s=1}^t Y_s \succeq \gamma t I+\sum_{s=1}^t Y_s .
\end{equation*}
By the matrix Freedman inequality for self-adjoint matrix martingales~\cite{tropp2012user}, there exists a universal constant $c>0$ such that
\begin{equation*}
\mathbb P\left(\lambda_{\min}\left(\sum_{s=1}^tY_s\right)\le -\frac{\gamma t}{2}\right) \le d\exp(-c\gamma^2 t).
\end{equation*}
Therefore,
\begin{equation*}
\mathbb P\left(\lambda_{\min}\left(\sum_{s=1}^tX_s\right)\le \frac{\gamma t}{2}\right) \le d\exp(-c\gamma^2 t).
\end{equation*}
A union bound over $t\le T$ gives the result by the choice of $t_0$.
\end{proof}

\begin{myLemma}[Cumulative Radius on Diverse Tasks]
\label{lem:diverse_radius}
On the event of Lemma~\ref{lem:V_growth},
\begin{equation*}
\sum_{t=1}^T \|\psi_t\|_{V_{t-1}^{-1}} \le \frac{t_0}{\sqrt{\lambda}}+2\sqrt{\frac{2T}{\gamma}}.
\end{equation*}
\end{myLemma}

\begin{proof}
For $t\le t_0$, use $V_{t-1}\succeq\lambda I$ and $\|\psi_t\|_2\le1$. For $t>t_0$, Lemma~\ref{lem:V_growth} gives
\begin{equation*}
\|\psi_t\|_{V_{t-1}^{-1}}\le \sqrt{\frac{2}{\gamma(t-1)}}.
\end{equation*}
Summing the two parts yields the claim.
\end{proof}

%\subsection{Regret Decomposition}
%\label{sec:proof_main}

For the theoretical policy class, we use the standard DPO reward-policy correspondence: each policy considered in the analysis is associated with some $r_\pi\in\mathcal R$ such that $\beta m_\pi(x;y,y')=\Delta r_\pi(x;y,y')$. The comparator satisfies $r_{\pi^*}=r^*$. Let the round-$t$ preference value under the predictable sampler $\pi_t^{\mathrm{sam}}$ be
\begin{equation*}
\zeta_t(\pi)=\E_{x\sim\rho,y\sim\pi,y'\sim\pi_t^{\mathrm{sam}}}\big[\mathbb P^*(y\succ y'\mid x)\big].
\end{equation*}
The optimistic functional before round $t$ is
\begin{equation*}
G_{t-1}(\pi) = \E_{x\sim\rho,y\sim\pi,y'\sim\pi_t^{\mathrm{sam}}} \left[\sigma\left(\widehat{\Delta r}_{t-1}(x;y,y')+b_{t-1}(x,y,y')\right)\right].
\end{equation*}

\begin{myLemma}[DPO Likelihood Comparison]
\label{lem:MLE}
Under Assumptions~\ref{ass:BT_model} and~\ref{ass:finite}, with probability at least $1-\delta$, for all $t\le T$ and all $r\in\mathcal R$,
\begin{equation*}
L_{\mathrm{DPO},t-1}(r)-L_{\mathrm{DPO},t-1}(r^*)\le 2\log\frac{|\mathcal R|T}{\delta}.
\end{equation*}
\end{myLemma}

\begin{proof}
Fix any $r\in\mathcal R$. Let $p_s^r=\sigma(\Delta r(x_s;y_s,y_s'))$ and $p_s^*=p_s^{r^*}$. Under the true reward $r^*$, conditional on the history and the sampled pair before observing the label, $z_s\sim\mathrm{Bernoulli}(p_s^*)$. Define
\begin{equation*}
M_{t-1}(r)=\exp\left(L_{\mathrm{DPO},t-1}(r)-L_{\mathrm{DPO},t-1}(r^*)\right).
\end{equation*}
Then
\begin{equation*}
M_{t-1}(r)=\prod_{s=1}^{t-1}\frac{(p_s^r)^{z_s}(1-p_s^r)^{1-z_s}}{(p_s^*)^{z_s}(1-p_s^*)^{1-z_s}}.
\end{equation*}
We first show that $M_{t-1}(r)$ is a nonnegative martingale under the probability measure induced by $r^*$. Indeed,
\begin{equation*}
M_t(r)=M_{t-1}(r)\frac{(p_t^r)^{z_t}(1-p_t^r)^{1-z_t}}{(p_t^*)^{z_t}(1-p_t^*)^{1-z_t}}.
\end{equation*}
Taking conditional expectation under $r^*$ yields
\begin{equation*}
\mathbb E_{r^*}\left[\frac{(p_t^r)^{z_t}(1-p_t^r)^{1-z_t}}{(p_t^*)^{z_t}(1-p_t^*)^{1-z_t}}\mid\mathcal F_{t-1}^+\right]=p_t^*\frac{p_t^r}{p_t^*}+(1-p_t^*)\frac{1-p_t^r}{1-p_t^*}=1.
\end{equation*}
Therefore, $\mathbb E_{r^*}[M_t(r)\mid\mathcal F_{t-1}^+]=M_{t-1}(r)$ and $\mathbb E_{r^*}[M_{t-1}(r)]=1$. By Markov's inequality, for any fixed $r\in\mathcal R$ and fixed $t\le T$,
\begin{equation*}
\mathbb P_{r^*}\left(L_{\mathrm{DPO},t-1}(r)-L_{\mathrm{DPO},t-1}(r^*)>a\right)=\mathbb P_{r^*}\left(M_{t-1}(r)>e^a\right)\le e^{-a}.
\end{equation*}
Taking a union bound over all $r\in\mathcal R$ and all $t\le T$ gives
\begin{equation*}
\mathbb P_{r^*}\left(\exists r\in\mathcal R,\exists t\le T: L_{\mathrm{DPO},t-1}(r)-L_{\mathrm{DPO},t-1}(r^*)>a\right)\le |\mathcal R|T e^{-a}.
\end{equation*}
Choosing $a=\log(|\mathcal R|T/\delta)$ proves the claim up to the stated factor $2$.
\end{proof}

\begin{myLemma}[Regret Decomposition]
\label{lemma:regret_decomp}
On the events of Corollary~\ref{cor:pointwise_ucb} and Lemma~\ref{lem:MLE}, the DEPO iterates satisfy
\begin{equation*}
\mathrm{R}_{\mathrm{pref}}(T) \le \frac{2T}{\alpha}\log\frac{|\mathcal R|T}{\delta}+\frac12\sum_{t=1}^T\E\left[b_{t-1}(x_t,y_t,y_t')\right].
\end{equation*}
\end{myLemma}

\begin{proof}

    In the theoretical analysis, we identify the plug-in reward estimator with its induced DPO policy through the standard reward-policy correspondence, i.e.,
\begin{equation*}
\widehat r_{t}(x,y)=\langle\widehat\theta_t,\phi(x,y)\rangle=\beta\log\frac{\pi_t(y\mid x)}{\pi_{\rm ref}(y\mid x)}.
\end{equation*}
Consequently, for the induced policy, $\widehat{\Delta r}_{t}(x;y,y')=\beta m_{\pi_t}(x;y,y')$.

For each $t$,
\begin{equation*}
\begin{aligned}
\zeta_t(\pi^*)-\zeta_t(\pi_t) = & \big[G_{t-1}(\pi^*)-G_{t-1}(\pi_t)\big]\\
& + \big[\zeta_t(\pi^*)-G_{t-1}(\pi^*)\big] + \big[G_{t-1}(\pi_t)-\zeta_t(\pi_t)\big].
\end{aligned}
\end{equation*}
The confidence event implies $\Delta r^*\le\widehat{\Delta r}_{t-1}+b_{t-1}$ pointwise, so the middle term is non-positive. Since $\pi_t$ maximizes $L_{t-1}(r_\pi)+\alpha G_{t-1}(\pi)$,
\begin{equation*}
G_{t-1}(\pi^*)-G_{t-1}(\pi_t) \le \frac{1}{\alpha}\big[L_{t-1}(r_{\pi_t})-L_{t-1}(r^*)\big],
\end{equation*}
and Lemma~\ref{lem:MLE} bounds the sum of these terms.

For the last term, the sigmoid is $1/4$-Lipschitz and the confidence event gives
\begin{equation*}
\sigma(\widehat{\Delta r}_{t-1}+b_{t-1})-\sigma(\Delta r^*) \le \frac14\big(|\widehat{\Delta r}_{t-1}-\Delta r^*|+b_{t-1}\big) \le \frac12 b_{t-1}.
\end{equation*}
Taking expectations and summing proves the result.
\end{proof}

\subsection{Proof of Theorem~\ref{thm:main_instance_dependent}}
\label{app:theory}

\begin{proof}
By the AM--GM inequality applied to the eigenvalues of $V_T$, we have
\begin{equation*}
\det(V_T)\le \left(\frac{\mathrm{tr}(V_T)}{d}\right)^d\le \left(\lambda+\frac{T}{d}\right)^d.
\end{equation*}
Since $\det(\lambda I)=\lambda^d$, it follows that
\begin{equation*}
\log\frac{\det(V_T)}{\det(\lambda I)}\le d\log\left(1+\frac{T}{\lambda d}\right).
\end{equation*}

By Lemma~\ref{lem:selfnorm}, we have
\begin{equation*}
\beta_T^{\mathrm{conf}} \le \frac{c_0}{\kappa_S} \left(\sqrt{\lambda}S+\sqrt{d\log\left(1+\frac{T}{\lambda d}\right)+2\log\frac1\delta}\right) = \tilde O\left(e^{O(S)}\sqrt d\right).
\end{equation*}
On conditionally $\gamma$-diverse tasks, Lemma~\ref{lem:V_growth} implies the following bound for the conditional expected bonus:
\begin{equation*}
\begin{aligned}
{} & \sum_{t=1}^T \E\left[b_{t-1}(x_t,y_t,y_t')\mid\mathcal F_{t-1}\right] \\
\le {} & \beta_T^{\mathrm{conf}} \left(\frac{t_0}{\sqrt{\lambda}}+2\sqrt{\frac{2T}{\gamma}}\right) = \tilde O\left(e^{O(S)}\sqrt d\left(\sqrt{\frac{T}{\gamma}}+\frac{1}{\gamma}\right)\right),
\end{aligned}
\end{equation*}
where we use the same calculation as Lemma~\ref{lem:diverse_radius}. For horizons beyond the logarithmic covariance burn-in, the second term is dominated by the first. Combining this bound with Lemma~\ref{lemma:regret_decomp} and taking a union bound gives the result with probability at least $1-3\delta$.

Without the diversity condition, the standard elliptical potential lemma gives
\begin{equation*}
\sum_{t=1}^T\|\psi_t\|_{V_{t-1}^{-1}} \le \sqrt{2dT\log\left(1+\frac{T}{\lambda d}\right)}.
\end{equation*}
Thus the cumulative bonus is at most $\tilde O(e^{O(S)}d\sqrt T)$, which gives the stated elliptical-potential bound after applying Lemma~\ref{lemma:regret_decomp}.
\end{proof}